\documentclass{article}
\usepackage{spconf,amsmath,graphicx}
\usepackage{amsthm,amssymb}
\usepackage{algorithm, algorithmic}

\usepackage{graphicx}
\usepackage{color}
\usepackage{epstopdf}
\usepackage{wrapfig}
\usepackage{cite}
\usepackage{hyperref}

\title{The Symmetry of a Simple Optimization Problem in Lasso Screening}
\name{Yun Wang and Peter J. Ramadge}
\address{Department of Electrical Engineering, Princeton University\\
Princeton, NJ 08544, USA}
\newcommand{\R}{\mathbb{R}} 
\newcommand{\by}{\times}
\newcommand{\norm}[2]{\lVert{#1}\rVert_{#2}}

\newcommand{\cN}{\mathcal{N}}

\newcommand{\vect}[1]{\mathbf{#1}}
\newcommand{\vzero}{\vect{0}}
\newcommand{\vq}{\vect{q}} 

\newcommand{\vb}{\vect{b}} 
\newcommand{\vc}{\vect{c}} 
\newcommand{\vd}{\vect{d}} 

\newcommand{\ve}{\vect{e}} 
\newcommand{\vu}{\vect{u}} 
\newcommand{\vv}{\vect{v}} 
\newcommand{\vw}{\vect{w}} 

\newcommand{\vx}{\vect{x}} 
\newcommand{\vy}{\vect{y}} 
\newcommand{\vz}{\vect{z}} 
\newcommand{\vn}{\vect{n}} 

\newcommand{\vth}{\vect{\boldsymbol{\theta}}} 
\newcommand{\vpsi}{\vect{\boldsymbol{\psi}}}  

\newcommand{\bmu}{\bar{\mu}}

\newcommand{\vwo}{\tilde{\vw}} 
\newcommand{\vto}{\tilde{\vth}} 

\newcommand{\FS}{\mathcal F} 
\newcommand{\RR}{\mathcal{R}} 
\newcommand{\DM}{{\color{black}D}} 

\newtheorem{proposition}{Proposition}
\newtheorem{theorem}{Theorem}
\newtheorem{lemma}{Lemma}
\newtheorem{corollary}{Corollary}

\newcommand{\dd}{n}
\newcommand{\nh}{m}

\newcommand{\simS}{\sim_{S_{\cN}} }
\newcommand{\facsp}{\R^{\dd}/S_{\cN}}

\begin{document}
%
\maketitle
\begin{abstract}
Recently dictionary screening has been proposed as an effective way to improve the computational efficiency of solving the lasso problem, which is one of the most commonly used method for learning sparse representations. To address today's ever increasing large dataset, effective screening relies on a tight region bound on the solution to the dual lasso. Typical region bounds are in the form of an intersection of a sphere and multiple half spaces. One way to tighten the region bound is using more half spaces, which however, adds to the overhead of solving the high dimensional optimization problem in lasso screening. This paper reveals the interesting property that the optimization problem only depends on the projection of features onto the subspace spanned by the normals of the half spaces. This property converts an optimization problem in high dimension to much lower dimension, and thus sheds light on reducing the computation overhead of lasso screening based on tighter region bounds.
\end{abstract}
\section{Introduction}
\label{sec:introduction}
The least squares problem with l-1 regularization, widely known as the lasso problem \cite{RobTib1996},
\begin{equation}
	\label{eq:lasso}
	\min_{\vw\in\R^{p}} \qquad  \frac{1}{2}\norm{\vx-B\vw}{2}^2
		+ \lambda \|\vw\|_{1},
\end{equation}
remains one of the most used method for obtaining sparse representations. As a nonlinear encoding of signal $\vx$, the solution $\vwo$ proves effective in a variety of subsequent decision tasks, \cite{Wright2009Robust, YWang2013SRC, DeepSparseCoding}.

Despite efficient algorithms for solving  \eqref{eq:lasso} exists \cite{Lee2007Efficient}, scalability to large datasets remains a major problem. Dictionary screening for the lasso was proposed to address this computational issue \cite{Tibshirani2010Strong, Ghaoui2012, Xiang2011Learning_b, Xiang2012Fast, YWang2013a, FCSS2016, WangThesis2015, LSS, GramfortICML15}. Given a target vector $\vx$, dictionary screening identifies a subset of features $\vb_i$ with $\vwo_i=0$. These features can then be removed from the dictionary, and a smaller lasso problem is solved to obtain a solution of the original problem. This can significantly reduce the size of the dictionary that is loaded into memory (provided to the lasso solver), and make finding a lasso solution faster.

As a first step in existing screening methods, one need to bound the solution $\vto$ to the dual problem of \eqref{eq:lasso} within a compact region $\RR$, and then solve the following optimization problem
\begin{equation}
    \label{eq:opt_screening}
\mu(\vb) = \max_{\vth \in \RR}\vth^T\vb.
\end{equation}
If a feature $\vb_i$ satisfies $\mu(\vb_i)<1$ and $\mu(-\vb_i)<1$, then it follows $\vwo_i=0$.
Commonly used region $\RR$ is in the form of the intersection of a sphere and multiple half spaces, i.e., $\RR = \{\vth:(\vth-\vq)^T(\vth-\vq) \leq r^2, \vn_k^T\vth-c_k \leq 0, \quad k = 1,\dots,\nh\}$. For instance, closed form solutions for $\nh=1, 2$ are available \cite{Xiang2012Fast, YWang2013a}.

Today's ever increasing size of big data not only makes solving \eqref{eq:lasso} slower, but loading the entire data into memory can be problematic in the first place. This places a demand on improving the effectiveness of screening, which in turn relies on a tight region bound $\RR$ for $\vto$: studies suggest that with a tighter $\RR$, the screening algorithm can reject more features. One simple way to obtain a tighter $\RR$ is by imposing a larger $\nh$. Empirical studies have shown that increasing the number of hyperplane constraints improves the rejection rate. For instance, \cite{YWang2013a} shows that when $\nh$ moves from $1$ to $2$, the rejection percentage increases from $22\%$ to $40\%$ for MNIST \cite{LeCun1998The-MNIST} dataset and from $60\%$ to $80\%$ for YALEBXF \cite{Georghiades2002From} dataset for a target $\lambda/\lambda_{\max}=0.4$. It is likely that by further increasing $\nh$, the screening performance can be further boosted. Finding more half space constraints is not a problem. Borrowing similar ideas from previous works, one can find $\nh$ half spaces from the codeword constraints of the dual problem in a greedy fashion \cite{YWang2013a}, or from the solutions to the previous $\nh$ solved instances in a sequential screening scheme \cite{Ghaoui2012, YWang2013b}.

However, the problem for a larger $\nh$ is the potential computation cost. For $\nh>2$, a clean closed-form solution is unlikely.  Even for $\nh=2$, the closed-form solution is already complicated. So with a larger $\nh$, one might eventually resort to numerical solutions, and solving the optimization problem \eqref{eq:opt_screening} in high data dimension with a more complex region $\RR$ can add to the overhead, which might compromise the benefits of screening.

It is thus of interest to study the properties of \eqref{eq:opt_screening}, with the hope of simplifying solving the optimization problem \eqref{eq:opt_screening}. Analysis in this paper shows that the solution to \eqref{eq:opt_screening} is a function of the projection of the features $\vb$ onto the subspace that is spanned by the normals of the half spaces in $\RR$. This shreds light on reducing the optimization problem \eqref{eq:opt_screening} from dimension $n$ to $m$. This has very practical implications, considering the scale of this dimension reduction: $n$ usually ranges in scale from a few hundreds (MNIST \cite{LeCun1998The-MNIST}) to more than a hundred thousands (NYT dataset \cite{Asuncion2007UCI-Machine}), while current $m$ is less than 10.

\section{Core Problem}

We formalize our problem as follows. Let $\vq, \vth \in \R^{\dd}$, $r, c_k >0$ and $\vn_k\in \R^{\dd}$ with $\|\vn_k\|_2=1$, $k=1,\dots,\nh$.
For given $\vb\in \R^{\dd}$, we  consider the simple optimization problem:
\begin{align}\label{eq:P0}
\begin{split}
\mu(\vb) = \max_{\vth\in \R^{\dd}}
&\quad \vth^T \vb\\
\text{s.t.}
&\quad (\vth-\vq)^T(\vth-\vq)-r^2 \leq 0\\
&\quad \vn_k^T\vth-c_k \leq 0 \qquad k = 1,\dots,\nh
\end{split}
\end{align}
The vector $\vb\in \R^{\dd}$ specifies the linear objective function and the parameters $\vq,r$
and $\vn_k,c_k$, $k=1,\dots,\nh$,
specify a spherical bound and $\nh$ half space constraints $\vn_k^T\vth\leq c_k$ on the feasible
points of the problem, respectively.
Using the change of variable $\vz=(\vth-\vq)/r$, problem  \eqref{eq:P0} can be simplified to:
%
\begin{align} \label{eq:P1}
\begin{split}
\bmu(\vb)=\max_{\vz\in \R^{\dd} }
& \quad \vz^T\vb \\
\text{s.t.}
&\quad  \vz^T\vz-1 \leq 0 \\
&\quad  \vn_k^T\vz+\psi_k \leq 0, \quad k = 1,\dots,\nh
\end{split}
\end{align}
where $\psi_k = (\vn_k^T\vq-c_k)/r$.
The solution of \eqref{eq:P0} can be obtained as
$\mu(\vb)=\vq^T\vb+r\bar{\mu}(\vb)$.
This problem has the same linear objective function specified by $\vb$.
However, $\vz$ is constrained to lie in the
intersection of the unit ball and the $\nh$ half spaces $\vn_k^T\vth +\psi_k\leq 0$.
This is illustrated in Fig.~\ref{fig:dome}.

We call the region $\DM_k = \{\vz\colon \vz^T\vz-1 \leq 0, \vn_k^T\vz+\psi_k \leq 0\}$,
consisting of the intersection of the unit ball and the half space $\vn_k^T\vz+\psi_k\leq 0$, a dome. The unit vector $\vn_k$ is the normal to the dome and the scalar
$\psi_k$ gives the distance from $0$ to the dome base.
This is illustrated in Fig.~\ref{fig:dome}.

Let $N=[\vn_1,\dots,\vn_{\nh}]$ and $\vpsi=[\psi_1,\dots,\psi_{\nh}]^T$.
Then we can write the core problem more concisely as:
\begin{align}\label{eq:P2}
\begin{split}
\bmu(\vb)=\max_{\vz\in \R^{\dd} }
& \quad \vz^T\vb \\
\text{s.t.}
&\quad  \vz^T\vz-1 \leq 0 \\
&\quad  N^T\vz+\vpsi \leq \vzero.
\end{split}
\end{align}
Problem \eqref{eq:P1} is parameterized by the pair $(N,\vpsi)$.
$N$ specifies the dome axes and the vector $\psi$ specifies their respective sizes.

\begin{figure}[!t]
  \centering
  \centerline{\includegraphics[width=0.75\linewidth]{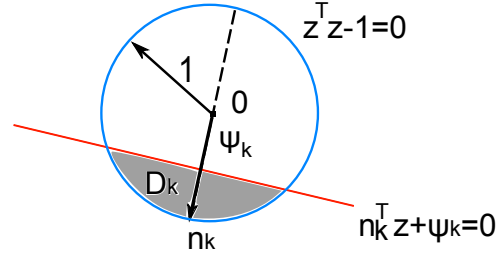}}
\caption{Illustration of the feasible set for Problem \eqref{eq:P1}, $k = 1, 2, \ldots, \nh$.}
\label{fig:dome}
\end{figure}

Let $O(\dd)$ denote the group of real $\dd\by \dd$
orthogonal matrices. If we transform the
parameters of Problem \eqref{eq:P1} by $Q$, then
$N$ maps to $QN$ and $\vpsi$ maps to $(N^TQ^TQ\vq -c)/r =\vpsi$,
where $\vc=(c_1,\dots,c_\nh)^T$. So the only change is to  the second constraint in Problem \eqref{eq:P1} which becomes $N^TQ^T\vz +\vpsi \leq \vzero$.
Setting $\vw=Q^T\vz$, the transformed problem
can be written as:
\begin{align*}
\begin{split}
\bmu_{Q}(\vb)=\max_{\vw\in \R^{\dd} }
& \quad \vw^T Q^T\vb \\
\text{s.t.}
&\quad  \vw^T\vw-1 \leq 0 \\
&\quad  N^T\vw+\vpsi \leq \vzero.
\end{split}
\end{align*}
Thus $\bmu_Q(\vb)=\bmu(Q^T\vb)$.
So if one ``rotates'' the problem by $Q$,
then  the solution $\bmu_Q(\vb)$
of the new problem is obtained by first inverse ``rotating'' $\vb$ via $Q^T$ and then computing
$\bmu(Q^T\vb)$. Intuitively, this is obvious.
But it indicates is that $\bmu$ must be determined by a function of the problem parameters, $(N,\psi)$, that is invariant under orthogonal transformations.
$\psi$ is invariant under $O(\dd)$.
So $\bmu$ must be a depend on a function of $N$
that is invariant under $O(\dd)$.
The simplest nontrivial function with this property is $N^TN$.
It is reasonable to expect that the function $\bmu$ depends only on the entries of $N^TN$ and $\psi$.
$N^TN$ determines the inter-dome configuration,
but not its overall orientation,
and $\psi$ specifies the respective dome sizes.

\section{A Symmetry Group of the Core Problem}
Let $\cN=\text{span}\{\vn_1, \dots, \vn_{\nh}\}$
and consider the subset $S_{\cN}$ of $O(\dd)$
defined by:
\begin{equation}\label{eq:defS_N}
S_{\cN}=\{R\colon R\in O(\dd), R\vn_k=\vn_k, k=1,\dots,\nh\}.
\end{equation}
It readily follows that for each $R\in S_{\cN}$,
the restriction of $R$ to $\cN$, denoted $R|\cN$, is the identity map.
Indeed if $\vx\in \cN$,
then $\vx=\sum_{i=1}^{\nh}\alpha_i \vn_i$
and hence
$R\vx =\sum_{i=1}^{\nh}\alpha_i R\vn_i =
\sum_{i=1}^{\nh}\alpha_i \vn_i= \vx$.


\begin{lemma}
$S_\cN$ is a subgroup of the orthogonal group $O(n)$.
\end{lemma}
\begin{proof}
One only needs to verify that $I\in S_{\cN}$, $R\in S_{\cN}$
implies $R^T\in S_{\cN}$ and that $S_{\cN}$ is closed
under matrix multiplication. The first and third properties
are clear. The second follows by noting that
$R\vn_k=\vn_k$ implies $R^TR\vn_k=R^T\vn_k$ and hence that
$\vn_k=R^T\vn_k$.
\end{proof}

\begin{lemma}
For each $R\in S_{\cN}$,
$R(\DM_k)=\DM_k$, $k=1,\dots,\nh$.
\end{lemma}

\begin{proof}
We first show that $R(\DM_k) \subset \DM_k$.
Let $\vd \in R(\DM_k)$.
So there exists $\ve\in \DM_k$ with $\vd=R\ve$.
We can uniquely write $\ve$ in the form
$\ve = \alpha\vn_k+\ve_0$, where $\vn_k^T\ve_0=0$.
Since $\ve\in \DM_k$, we have $\|\ve\|_2\leq 1$
and hence $\|\alpha \vn_k + \ve_0 \|_2=
|\alpha| + \| \ve_0 \|_2 \leq 1$. In addition, $\vn_k^T\ve+\psi_k\leq 0$.
Hence $\vn_k^T(\alpha \vn_k + \ve_0)+\psi_k
=\alpha +\psi_k \leq 0$.
Now
    \begin{equation*}
    \vd = R\ve  = R(\alpha\vn_k+\ve_0)
          = \alpha(R\vn_k)+R\ve_0
          = \alpha\vn_k+R\ve_0
    \end{equation*}

So $\|d\|_2=\|\alpha\vn_k+R\ve_0\|_2
= |\alpha|+\|\ve_0\|_2 \leq 1$ and
$\vn_k^T\vd+\psi_k = \vn_k^T(\alpha\vn_k+R\ve_0)
+\psi_k
= \alpha +\psi_k \leq 0$. Thus $\vd\in \DM_k$.

We now show that $\DM_k \subset R(\DM_k)$.
Let $\vd \in \DM_k$. So $\|\vd\|_2\leq 1$
and $\vn_k^T\vd+\psi_k\leq 0$.
Write $\vd$ uniquely in the form
$\vd = \alpha \vn_k+\vd_0$,
with $\vn_k^T\vd_0=0$.
Then $\|\vd\|_2 =\|\alpha \vn_k+\vd_0\|_2
=|\alpha| +\|\vd_0\|_2 \leq 1$ and
$\vn_k^T\vd+\psi_k = \alpha +\psi_k\leq 0$.
Now
\begin{align*}
\vd & = RR^T\vd = RR^T(\alpha\vn_k+\vd_0) = R(\alpha R^T\vn_k+R^T\vd_0)\\
& =R(\alpha\vn_k+R^T\vd_0).
\end{align*}
Set $\ve = \alpha \vn_k+R^T\vd_0$.
So  $\vd = R\ve$.
It only remains to show that $\ve\in \DM_k$.
This follows by noting that
$\|\ve\|_2 =\|\alpha \vn_k+R^T\vd_0\|_2
= |\alpha| + \|\vd_0\|_2 \leq 1$
and
$\vn_k^T\vd+\psi_k = \alpha+\psi_k \leq 0$.
\end{proof}

\begin{lemma}
For each $R\in S_{\cN}$,
    $R(\cap_k\DM_k) = \cap_k R(\DM_k)$.
\end{lemma}
\begin{proof}
We first show that  $R(\cap_k\DM_k)
\subseteq  \cap_k R(\DM_k)$.
Let $\vd \in R(\cap_k\DM_k)$.
Then there exists $\ve \in \cap_k\DM_k$ such that $\vd =R\ve$. Since $\ve\in \DM_k$, $k=1,\dots,\nh$, $\vd \in R(\DM_k)$, $k=1,\dots,\nh$. Thus $\vd \in \cap_kR(\DM_k)$.

Now we show that $R(\cap_k\DM_k)
\supseteq  \cap_k R(\DM_k)$.
Let $\vd \in \cap_k R(\DM_k)$.
So $\vd \in R(\DM_k)$, $k=1,\dots,\nh$.
Hence there exists $\ve_k \in \DM_k$ such that
$R(\ve_k)=\vd$, $k=1,\dots,\nh$. The invertibility of $R$ implies that $\ve_k =\ve=R^T \vd$.
So $\vd=R\ve$ with $\ve \in \cap_k \DM_k$.
Thus $\vd \in R(\cap_k \DM_k)$.
\end{proof}

The following result follows immediately
from the two previous lemmas.

\begin{lemma} \label{lem:RF=F}
For each $R\in S_{\cN}$,
$R(\cap_k\DM_k)
=\cap_k\DM_k$.
\end{lemma}

Now $\cap_{k=1}^{\nh} \DM_k$ is the set $\FS$ of feasible points for problem \eqref{eq:P1}. So Lemma \ref{lem:RF=F} indicates that $\FS$ is invariant under the group $S_{\cN}$.

The symmetry group of $\FS$ is the subgroup $S_{\FS}$
of the orthogonal group $O(\dd)$  with the property that $R\in S_{\FS}$ if and only if $R(\FS)=\FS$.
Hence, by Lemma \ref{lem:RF=F},
$S_{\cN}$ is a subset of the symmetry group of $\FS$.
In specific cases, $S_{\cN}$ can be a strict subset of
$S_{\FS}$ due to symmetries among the domes $\DM_k$.
For example, consider $\nh=2$ with $\vn_1\neq \vn_2$
but $\psi_1=\psi_2$. In this case, the domes $\DM_1$ and $\DM_2$ are identical except one is centered along
$\vn_1$ and the other along $\vn_2$.
A reflection about the hyperplane formed as the
perpendicular bisector of the line joining $\vn_1$ and $\vn_2$, maps $\DM_1$ to $\DM_2$ and
vice versa. It is thus a symmetry of $\DM_1\cap \DM_2$
but does not leave $\vn_1, \vn_2$ invariant and hence is not in $S_{\cN}$. On the other hand this symmetry is
not structurally stable in the sense that an arbitrarily small perturbation of the parameters will result in $\psi_1\neq \psi_2$ and hence in the loss of this symmetry.

%

The group $S_{\cN}$ splits $\R^{\dd}$ into
mutually exclusive equivalences classes with
\begin{equation}
\vz_1 \simS \vz_2 \quad \Leftrightarrow \quad
\vz_2=R\vz_1, \text{ some } R\in S_{\cN}
\end{equation}
The equivalence class of $\vz\in \R^{\dd}$, denoted $[\vz]$, is the set all points equivalent to $\vz$.
Since the elements of $S_{\cN}$ are orthogonal,
elements in the same equivalence class have the same norm. So $[\vz]$ is a subset of the sphere of radius $\|\vz\|_2$.

Let $\facsp = \{[\vx]\colon \vx\in \R^{\dd}\}$ denote
the set of all equivalences classes defined by the action of $S_{\cN}$ on $\R^{\dd}$.

\section{The Invariance of $\bmu$ Under $S_{\cN}$}

We now show that the value of $\bmu(\vx)$
is the same for all elements of $[\vx]$.


\begin{proposition}\label{pro:invmu}
For each $\vx,\vy\in \R^{\dd}$, if
$\vx \simS \vy$, then
$\bmu(\vx)=\bmu(\vy)$.
 \end{proposition}
\begin{proof} Let $R\in S_{\cN}$.
From the definition of $\bmu(\vx)$ we have
\begin{align*}
\bmu(R\vx)
& = \max_{\vz \in \cap_k \DM_k}~\vz^TR\vx
    = \max_{\vz\in \cap_k \DM_k}~(R^T\vz)^T\vx \\
& = \max_{\vw \in R^T(\cap_k\DM_k)} \quad \vw^T\vx
\end{align*}
Since $R^T\in S_{\cN}$,
we can use Lemma \ref{lem:RF=F} to replace
$R^T(\cap_k\DM_k)$ by $\cap_k \DM_k$, yielding
\begin{equation*}
\bmu(R\vx)
= \max_{\vw\in \cap_k \DM_k} \vw^T\vx
= \bar{\mu}(\vx)
.
\end{equation*}
\end{proof}

Proposition \ref{pro:invmu} allows us to define a function $f_{\bmu}\colon \facsp\rightarrow \R$
by $f_{\bmu}([\vx]) =\bmu(\vx)$.
In this sense, the value of $\bmu$ at a point $\vx$
is determined by just knowing the equivalence class of $\vx$.


Suppose the columns of $V=[\vv_1,\dots,\vv_{\nh}]$ are in $\cN$.
Then for each $R\in S_{\cN}$, we have $RV=V$.
Hence for any $R\in S_{\cN}$ and $\vx\in \R^{\dd}$,
\begin{equation}\label{eq:VTz=VTRz}
V^T (R \vx) = (R^TV)^T \vx = V^T\vx.
\end{equation}
In particular, by looking at the $k$-th component of
\eqref{eq:VTz=VTRz}, we see that $\vv_k^T R \vx=\vv_k^T\vx$, $k=1,\dots,\nh$.
A special case is $V=N=[\vn_1, \vn_2, \ldots, \vn_m]$.
Thus for all $R\in S_{\cN}$ and $\vx\in \R^{\dd}$,
\begin{equation}\label{eq:NTz=NTRz}
N^T (R \vx) = (R^TN)^T \vx = N^T\vx.
\end{equation}
and $\vn_k^T R \vx=\vn_k^T\vx$, $k=1,\dots,\nh$.

\section{Orthogonal Projection onto $\cN$}


For $\vz\in \R^{\dd}$, let $\widehat{\vz}$ denote the  point obtained by orthogonally projecting $\vz$ onto the subspace $\cN$.
Let the columns of $V=[\vv_1, \vv_2, \ldots, \vv_m]$ form a basis for $\mathcal{N}$.
Then $\widehat{\vz}$ is given by,
\begin{equation}
\widehat{\vz} = V(V^TV)^{-1}V^T\vz
\end{equation}
In particular, if $V=[\vv_1,\dots,\vv_{\nh}]$ has orthonormal columns, then
$\widehat{\vz}=V V^T \vz$.

%

We first show that all points in $[\vz]$ have
the same orthogonal projection $\widehat{\vz}$.

\begin{lemma}\label{lem:z1sz2}
If $\vz_1 \simS \vz_2$,  then $\widehat{\vz}_1 =\widehat{\vz}_2$.
\end{lemma}

\begin{proof}
Let the columns of $V$ be a basis for $\cN$.
Then
$\widehat{\vz}_k = V(V^TV)^{-1}V^T\vz_k$, $k=1,2$.
Since $\vz_1 \simS \vz_2$, there exists
$R\in S_{\cN}$ such that $\vz_2=R\vz_1$.
Hence $\widehat{\vz}_2 = \widehat{R\vz_1}=
V(V^TV)^{-1}V^TR\vz_1$.
Thus by \eqref{eq:VTz=VTRz},
$\widehat{\vz}_1=\widehat{\vz}_2$.
\end{proof}

Lemma \ref{lem:z1sz2}
indicates that $\widehat{\vz}$
is determined by knowing the equivalence
class of $\vz$.

We can refine this slightly further by making the basis for $\cN$ explicit.

\begin{lemma}\label{lem:cz1=cz2}
Let the columns of $V$ form a basis for $\cN$.
If $\vz_1 \simS \vz_2$,  then $V^T\vz_1 =V^T\vz_2$.
\end{lemma}

\begin{proof}
By Lemma \ref{lem:z1sz2} we know that the vectors
\begin{equation*}
\widehat{\vz}_k = V(V^TV)^{-1}V^T\vz_k, \quad k=1,2,
\end{equation*}
are equal.
Hence $V^T\vz_1=V^T\vz_2$.
\end{proof}

For a fixed basis $V$ for $\cN$, Lemma \ref{lem:cz1=cz2}
shows that $V^T\vz$ is determined by knowing the equivalence class of $\vz$.

Next we show that $\widehat{\vb}$ together with
$C=\|\vb\|_2$, uniquely determine $[\vb]$ and hence
$\bmu(\vb)$.

\begin{lemma}\label{lem:hz1=hz2}
If $\|\vz_1\|_2=\|\vz_2\|_2$ and
$\widehat{\vz}_1=\widehat{\vz}_2$,
then $\vz_1 \simS \vz_2$ and $\bmu(\vz_1)=\bmu(\vz_2)$.
\end{lemma}

\begin{proof}
Let $\|\vz_1\|_2=\|\vz_2\|_2=C$ and
$\widehat{\vz}_1 = \widehat{\vz}_2 = \widehat{\vz}$.
Decompose each $\vz_k$ as the sum of its projection and the corresponding
orthogonal residual. So $\vz_k  = \widehat{\vz}+\vz_k^0$
and $\norm{\vz_k^0}{2}^2
= C^2 - \norm{\widehat{\vz}}{2}^2$, $k=1,2$.

Now $R\vz_2=R\widehat{\vz}+R\vz_2^0 = \widehat{\vz}+R\vz_2^0$.
So to prove $\vz_1\simS \vz_2$, we only need to show
that there exists $R \in S_\mathcal{N}$ with
$\vz_1^0 = R\vz_2^0$.
We construct such a $R$  below.

Let the columns of $V\in \R^{\dd\by \nh}$ form an orthonormal basis for $\cN$.
Set $\vu_{\nh+1}=\vz_1^0/\|\vz_1^0\|_2$.
Then select $U_2\in \R^{\dd\by (\dd-\nh-1)}$
so that $U=[V ~\vu_{\nh+1}~U_2]$
is an orthogonal matrix.
Similarly, set $\vv_{\nh+1} = \vz_2^0/\|\vz_2^0\|_2$
and select  $V_2\in \R^{\dd\by (\dd-\nh-1)}$ so that
$W=[V~\vv_{\nh+1}~ V_2]$ is an orthogonal
matrix. Now set $R=UW^T$. Since $U$ and $W$
are orthogonal, so is $R$. Moreover, for $\vz \in \cN$,
$R\vz= UW^T \vz = VV^T\vz = \vz$.
So $R\in S_{\cN}$. Finally, using the fact that
$\|\vz_1^0\|_2=\|\vz_2^0\|_2$ we have
\begin{equation*}
R\vz_2^0 = UW^T\vz_2^0
= \frac{\vz_1^0}{\|\vz_1^0\|_2}  \|\vz_2^0\|_2 = \vz_1^0
\end{equation*}
The fact that $\bmu(\vz_1)=\bmu(\vz_2)$
then follows by Proposition \ref{pro:invmu}.
%
\end{proof}

By Lemma \ref{lem:hz1=hz2}, if we fix a value for the norm, say $C=1$, then $\bmu(\vb)$ is uniquely determined by $\widehat{\vb}$.
This can be slightly refined by making the basis for $\cN$ explicit.

\begin{lemma}\label{lem:VTz1=VTz2}
Let the columns of $V$ form a basis for $\cN$.
If $\|\vb_1\|_2=\|\vb_2\|_2$ and
$V^T\vb_1 = V^T\vb_2$,
then
$\vb_1 \simS \vb_2$ and
$\bmu(\vb_1)=\bmu(\vb_2)$.
\end{lemma}

\begin{proof}
$\widehat{\vz}_k=V(V^TV)^{-1}V^T\vz_k$, $k=1,2$.
Hence if $V^T\vz_1 = V^T\vz_2$, then
$\widehat{\vz}_1=\widehat{\vz}_2$. The result then
follows by Lemma \ref{lem:hz1=hz2}.
\end{proof}


Putting the above observations together, gives the following result.

\begin{theorem} \label{thm:mubarprojection}

There exists a function $g\colon \cN \rightarrow \R$
such that for unit norm $\vb\in \R^{\dd}$,
$\bmu(\vb) = g(\widehat{\vb})$.
\end{theorem}

\begin{proof}
Given $\widehat{\vb}$ and the fact that $\vb$ has unit norm,
by Lemma \ref{lem:hz1=hz2} we have $[\vb]=f_{[1]}(\widehat{\vb})$. Then we use Proposition \ref{pro:invmu}
to determine $\bmu(\vb)=f_{\bmu}([\vb])$.
Thus the desired function is
the composition $g=f_{\bmu}\cdot f_{[1]}$.
\end{proof}

The following Corollary is a natural consequence of Theorem \ref{thm:mubarprojection}.

\begin{corollary}
Let the columns of $V=[\vv_1, \vv_2, \ldots, \vv_m]$ form a basis for $\mathcal{N}$.
Then there exists a function $h\colon\R^{\nh} \rightarrow \R$
such that for unit norm $\vb\in \R^{\dd}$,
\begin{equation}
\bar{\mu}(\vb) = h(\vv_1^T\vb, \vv_2^T\vb, \ldots, \vv_m^T\vb).
\end{equation}
\end{corollary}

\begin{proof}
$\widehat{\vb}=V(V^TV)^{-1}(V^T\vb)$. Thus from
\begin{equation*}
V^T\vb=(\vv_1^T\vb,\dots,\vv_{\nh}^T\vb)^T,
\end{equation*}
we can compute $\widehat{\vb}$.
Then using the fact that $\vb$ is unit norm and Theorem
\ref{thm:mubarprojection}, the desired function is the composition $h(V^T\vb)= (g \cdot V(V^TV)^{-1})(V^T\vb)$.
\end{proof}
\section{The Simplified Problem}
Having shown that Problem \eqref{eq:P0} and \eqref{eq:P1} depends only on the projection of $\vb$, now we present the simplified problem.

Let the columns of $V=[\vv_1, \ldots, \vv_m]$ form an ON basis for $\cN$, and the columns of $U=[\vu_1, \ldots, \vu_{n-m}]$ form an ON basis for $\cN^\perp$. De-compose $\vz$ and $\vb$ as the sum of its projection and the corresponding orthogonal residual, i.e., $\vz=\widehat{\vz}+\vz^0$, and $\vb=\widehat{\vb}+\vb^0$. We have $t_\vz=V^T\vz$, $t_\vz^0=U^T\vz$, $t_\vb=V^T\vb$, and $t_\vb^0=U^T\vb$. Problem \eqref{eq:P1} is then equivalent to the following,
\begin{align} \label{eq:sP1}
\begin{split}
\max_{t_\vz\in \R^{\nh}, t_\vz^0 \in \R^{\dd-\nh} }
& \quad t_\vz^Tt_\vb+ (t_\vz^0)^Tt_\vb^0\\
\text{s.t.}
&\quad  t_\vz^Tt_\vz +(t_\vz^0)^Tt_\vz^0 \leq 1 \\
&\quad  N^TVt_\vz+\vpsi \leq \vzero.
\end{split}
\end{align}

Let $A=N^TV\in\R^{m\times m}$, and define function $k(\vx)=\sqrt{1-\norm{\vx}{2}^2}$. Then it follows trivially, the above problem is equivalent to,
\begin{align} \label{eq:ssP1}
\begin{split}
\max_{t_\vz\in \R^{\nh}}
& \quad t_\vz^Tt_\vb+ k(t_\vz)k(t_\vb)\\
\text{s.t.}
&\quad  At_\vz+\vpsi \leq \vzero.
\end{split}
\end{align}
This is a $m$-dimension quadratic optimization problem with $m$ linear constraints. If we further denote $t_\vz^e=(t_\vz, k(t_\vz))^T$, $t_\vb^e=(t_\vb, k(t_\vb))^T$, under the assumption $\vb \notin \cN$, with KKT condition we get another equivalent problem in dimension $\nh+1$,
\begin{align} \label{eq:ssP1}
\begin{split}
\max_{t_\vz^e\in \R^{\nh+1}}
& \quad (t_\vz^e)^Tt_\vb^e\\
\text{s.t.}
&\quad  (t_\vz^e)^Tt_\vz^e \leq 1 \\
&\quad  [A \:\vzero] t_\vz^e+\vpsi \leq \vzero.
\end{split}
\end{align}

\section{Conclusion}
In this paper, we studied a simple optimization problem that is the key in lasso screening. The $\dd$-dimension optimization problem has a linear objective function with a feasible set that is the intersection of a spherical region and $\nh$ half spaces. Incorporating more half spaces gives hope of stronger screening performance, but in the meantime may increase the computational cost of screening. Analysis in our paper demonstrates that the optimization problem is a function of the projection of a feature onto a subspace spanned by the normals of the $\nh$ half spaces. This result reduces the dimension of the problem from $n$ (dimension of data points) to $m$ (the number of half spaces), which is a reduction of several orders of magnitude. The simplified problem is of same form to the original problem: a linear objective with linear and quadratic constraints, which implies that the dimensionality reduction can lead to the reduction in computational cost. This sheds light on improving the effectiveness of screening by using a tighter region bound while at the same time keeping its computational cost at bay.


\bibliographystyle{IEEEbib}

{\small\bibliography{Lassobib,books,Screening2012,musicbib}}

\end{document}